\documentclass[preprint,review,a4paper,11pt]{elsarticle}

\usepackage{graphicx}
\usepackage{amssymb}
\usepackage{lineno}
\usepackage{comment}
\usepackage{graphicx}
\usepackage{latexsym}
\usepackage{amsmath, amssymb, amsthm}  % align
\usepackage{color}
\usepackage{dsfont} % Indicatrice
\usepackage{stmaryrd} % Double brackets
\usepackage[linesnumbered, ruled, noend]{algorithm2e}
\usepackage[small]{caption}
\usepackage{multirow}

\newcommand{\per}{\mathrm{PER}}
\newcommand{\mer}{\mathrm{MER}}
\newcommand{\mmer}{\mathrm{MMER}}

\newtheorem{defi}{Definition}
\newtheorem{prop}{Proposition}

\journal{European Journal of Operational Research}

\begin{document}

\begin{frontmatter}


\title{Bayesian preference elicitation for multiobjective combinatorial optimization}



\author[sorbonne]{Nadjet Bourdache}
\author[sorbonne]{Patrice Perny}
\author[sorbonne]{Olivier Spanjaard}

\address[sorbonne]{Sorbonne Universit\'e, CNRS, LIP6, F-75005 Paris, France,\\ email: name.surname@lip6.fr}

\begin{abstract}
We introduce a new incremental preference elicitation procedure able to deal with noisy responses of a Decision Maker (DM). The originality of the contribution is to propose a Bayesian approach for determining a preferred solution in a multiobjective decision problem involving a \emph{combinatorial} set of alternatives. We assume that the preferences of the DM are represented by an aggregation function whose parameters are unknown and that the uncertainty about them is represented by a density function on the parameter space. Pairwise comparison queries are used to reduce this uncertainty (by Bayesian revision). The query selection strategy is based on the solution of a mixed integer linear program with a combinatorial set of variables and constraints, which requires to use columns and constraints generation methods. Numerical tests are provided to show the practicability of the approach. 
\end{abstract}


\end{frontmatter}

%\linenumbers

% ---------------------------------------------------------------------
% ---------------------------------------------------------------------
% ---------------------------------------------------------------------

\section{Introduction}

The increasing complexity of problems encountered in applications is  a permanent motivation for the development of intelligent systems for human decision support.
Among the various difficulties to overcome for decision making in complex environments we consider here three sources of complexity that often coexist in a decision problem:  1) the combinatorial nature of the set of feasible alternatives  2) the fact that multiple points of view, possibly conflicting, about the value of solutions may coexist, 3) the need of formulating recommendations that are tailored to the objectives and preferences of users and that takes into account the uncertainty in preference elicitation (due to possible mistakes in the responses of users to preference queries).

The first difficulty occurs as soon as the solutions to be compared are characterized by the combinations of elementary decisions. This is the case for instance for the selection problem of an optimal subset within a reference set, under a budget constraint (a.k.a. knapsack problem) where a solution is characterized by elementary decisions concerning items of the reference set. This difficulty prevents the explicit evaluation of all solutions and the determination of the best option requires implicit enumeration techniques. The second difficulty appears in multiagent decision contexts when the agents have different individual value systems or objectives leading to possibly conflicting preferences. It also appears in single-agent decision contexts when the alternatives are assessed w.r.t. different criteria. Finally, it appears in decision under uncertainty when several scenarios that have different impacts on the outcomes of the alternatives are considered. In all these situations, preference modeling requires the definition of multiple objectives to be optimized simultaneously. The combination of difficulties 1 and 2 is at the core of multiobjective combinatorial optimization \cite{Ehrgott05}.

Let us now come to the third difficulty. The coexistence of multiple objectives makes the notion of optimality subjective and requires additional preference information to be collected from the users in order to discriminate between Pareto-optimal solutions. In multiobjective decision problems, the ``optimal'' solution fully depends on the relative importance attached to the different objectives under consideration and on how performances are aggregated. A standard tool used to generate compromise solutions tailored to the decision maker (DM) value system is to optimize a parameterized aggregation function summarizing the performance vector of any solution into a scalar value. This makes it possible to reformulate the initial problem as a single-objective optimization problem (see e.g., \cite{Steuer86}). However, a precise specification of the preference parameters (e.g., weighting coefficients), prior to the exploration of the set of alternatives, may be cumbersome because it requires a significant amount of preference information. To overcome this problem, incremental decision procedures aiming to integrate and combine the elicitation of preference parameters and the exploration of the set of feasible solutions are appealing (alternatively, one may also consider the approach consisting in computing the non-dominated solutions according to a scalarizing
function whose parameters are only partially specified \cite{Kaddani17}). They make it possible to focus the elicitation burden on the information that is really useful to separate competing solutions during the optimization process, and this significantly reduces the number of queries asked to the user.

In the fields of operations research and  artificial intelligence, numerous contributions have addressed the problem of incrementally eliciting preferences. A first stream of research concerns preference elicitation for decision making in explicit sets (i.e., non-combinatorial problems), to assess multiattribute utility functions \cite{WhiteSD84}, weights of criteria in aggregation functions \cite{BenabbouPV17}, multicriteria sorting models \cite{ozpeynirci2018interactive}, utility functions for decision making under risk \cite{Chajewska00,WangBoutilier03,hines10,PernyVB16}, or individual utilities in collective decision making \cite{LuB11}.
Preference elicitation for decision support on combinatorial domains is a challenging issue that has also been studied in various contexts
such as constraint satisfaction \cite{gelain10}, matching under preferences \cite{DrummondB14}, sequential decision making under risk \cite{Regan11,WengP13,gilbertSVW15,BenabbouP17}, and multiobjective combinatorial optimization \cite{branke2016using,Benabbou18,BourdacheP19}.

Almost all incremental elicitation procedures mentioned above proceed by progressive reduction of the parameter space until an optimal decision can be identified. At every step of the elicitation process, a preference query is asked to the DM and the answer induces a constraint on the parameter space, thus a polyhedron including all parameter values compatible with the DM's responses is updated after each answer (\emph{polyhedral method} \cite{toubia04}). Queries are selected to obtain a fast reduction of the parameter space, in order to enforce a fast determination of the optimal solution. However, such procedures do not offer any opportunity to the DM to revise her opinion about alternatives and the final result may be sensitive to errors in preference statements.

A notable exception in the list of contributions mentioned above is the approach proposed by Chajewska et al. \cite{Chajewska00}. The approach relies on a prior probabilistic distribution over the parameter space and uses preference queries over gambles to update the initial distribution using Bayesian methods. It is more tolerant to errors and inconsistencies over time in answering preference queries. The difficulties with this approach may lie in the choice of a prior distribution and in the computation of Bayesian updates at any step of the procedure. A variant, proposed in \cite{GuoS10}, relies on simpler questions under certainty, so as to reduce the cognitive load.
%Furthermore, an incremental approach has been proposed to elicit the parameters of a multiattribute utility function \cite{saure19}, that updates an ellipsoidal credibility region computed from a multivariate normal distribution over the space of parameters.

 \paragraph{Motivation of the paper} As far as we know, the works mentioned in the last paragraph has not been extended for decision making on combinatorial domains. Our goal here is to fill the gap and to propose a Bayesian approach for determining a preferred solution in a multiobjective combinatorial optimization problem. The main issue in this setting is the determination of the next query to ask to the DM, as there is an exponential number of possible queries (due to the combinatorial nature of the set of feasible solutions).
 
 \paragraph{Related work} Several recently proposed Bayesian preference elicitation methods may be related to our work.\\ [0.5ex]
 -- Saur\'e and Vielma \cite{saure19} proposed an error tolerant variant of the polyhedral method, where the polyhedron is replaced by an ellipsoidal credibility region computed from a multivariate normal distribution on the parameter space. This distribution, and thus the ellipsoidal  credibility region, is updated in a Bayesian manner after each query. In contrast with their work, where the set of alternatives is explicitly defined, our method applies on implicit sets of alternatives. Besides, although our method also involves a multivariate normal density function on the parameter space, our query selection strategy is based on the whole density function and not only on a credibility region.\\ [0.5ex]
 -- Vendrov et al. \cite{VeLHB20} proposed a query selection procedure able to deal with large sets of alternatives (up to hundreds of thousands) based on \emph{Expected  Value Of Information} (EVOI). The EVOI criterion consists in determining a query maximizing the expected utility of the recommended alternative conditioned on the DM's answer (where the probability of each answer depends on a response model, e.g. the logistic response model). However, the subsequent optimization problem becomes computationally intractable with a large set of alternatives. The authors consider a continuous relaxation of the space of alternatives that allows a gradient-based approach. Once a query is determined in the relaxed space, the corresponding pair of fictive alternatives is projected back into the space of feasible alternatives. In addition, a second contribution of the paper is to propose an elicitation strategy based on \emph{partial comparison queries}, i.e. queries involving partially specified multi-attribute alternatives, which limits the cognitive burden when the number of attributes is large. We tackle here another state-of-the-art query selection strategy that aims at minimizing the max regret criterion (instead of maximizing the EVOI criterion), a popular measure of recommendation quality.\\ [0.5ex]
 -- In a previous work \cite{BourdachePS19}, we introduced an incremental elicitation method based on Bayesian linear regression for assessing the weights of rank-dependent aggregation functions used in decision theory (typically OWA and Choquet integrals). The query selection strategy we proposed is based on the min max regret criterion, similarly to the one we use in the present work. However, the method can only be applied to \emph{explicit} sets of alternatives and does not scale to combinatorial domains. The computation of regrets in the provided procedure (in order to determine the next query) requires indeed the enumeration of all possible pairs of solutions for each query, which is impractical if the set of solutions is combinatorial in nature. The change in scale is considerable. For illustration, instances involving 100 alternatives were considered in the numerical tests of our previous work \cite{BourdachePS19} while in the multi-objective knapsack instances under consideration in Section 4 of the present paper, there are about $2^{99}$ feasible solutions, among which several millions are Pareto optimal. In order to scale to such large problems, we propose a method based on mixed integer linear programming that allows us to efficiently compute MER and MMER values on combinatorial domains.
 
 %Bourdache et al.~\cite{BourdachePS19} proposed an incremental elicitation method based on Bayesian linear regression for assessing the weights of rank-dependent aggregation functions used in decision theory (typically OWA and Choquet integrals).  Nevertheless, their approach is not able to deal with large, or even combinatorial, sets of alternatives. The method we present in this paper makes it possible to extend their incremental elicitation procedure to combinatorial sets of alternatives.

\paragraph{Organization of the paper} In Section 2, we describe the incremental elicitation procedure proposed in \cite{BourdachePS19} for the determination of an optimal solution over \emph{explicit} sets of solutions and we point out the main issues to overcome to extend the approach to combinatorial sets of solutions. Section 3 is devoted to the method we propose to compute expected regrets on combinatorial domains. The results of numerical tests we carried out are presented in Section 4, that show the operationality of the proposed procedure.

\section{Incremental elicitation}
\label{sec:incremental}

%We describe here the incremental Bayesian elicitation procedure we use in the paper. We first provide a general overview of the approach proposed by Bourdache et al.~\cite{BourdachePS19} on an \emph{explicit set} of alternatives, then we highlight the issues encountered for extending their approach to multiobjective \emph{combinatorial optimization} problems. 

We first provide a general overview of  the incremental Bayesian elicitation procedure on an \emph{explicit set} and then discuss the extension to cope with \emph{combinatorial optimization} problems. 

Let $\mathcal{X}$ denote the set of possible solutions. Since we are in the context of multiobjective optimization, we assume that a utility vector $u(x)\!\in\!\mathbb{R}^n$ is assigned to any solution $x\!\in\! \mathcal{X}$. Then we consider the problem of maximizing over $\mathcal{X}$ a scalarizing function of the form $f_w(x)\!=\!\sum_{k=1}^n w_k g_k(u(x))$ where $w_k$ are positive weights and $g_k\!:\!\mathbb{R}^n\!\rightarrow\!\mathbb{R}$ are \emph{basis functions} \cite{bishop2006pattern} (introduced to extend the class of linear models to nonlinear ones). %The basis functions have been introduced to extend the class of linear models by considering linear combinations of nonlinear functions. %defines a permutation of $(u_1(x), \dots, u_n(x))$. 
For simplicity, the reader can assume that $g_k(u(x))\!=\!u_k(x)$, i.e., the $k$-th component of $u(x)$, and $w_k$ is the (imperfectly known) weight of criterion $k$ in the following. %, and tIf $f_w$ is linear then $g_k(x)\!=\!u_k(x)$.
At the start of the procedure, a prior density function $p$ is associated to the parameter space $W\!=\!\{w\in[0,1]^n|\sum_kw_k\!=\!1\}$, where the unknown weighting vector $w$ takes value. Then, at each step, the DM responds to a pairwise comparison query and, based on this new preference information, the density function is updated in a Bayesian manner. The aim %of an incremental elicitation procedure 
is, in a minimum number of queries, to acquire enough information about the weighting vector $w$ to be able to recommend a solution $x\!\in\!\mathcal{X}$ that is near optimal.
We present here the three main parts of the decision process: query selection strategy, Bayesian updating after each query and stopping condition.

\subsection{Query selection strategy} \label{subsec:qss}
At each step of the algorithm, a new preference statement is needed to update the density function on $W$. In order to select an informative query, we use an adaptation of the \emph{Current Solution Strategy} (CSS) introduced in \cite{BoutilierPPS06} and based on regret minimization. In our probabilistic setting, regrets are replaced by \emph{expected} regrets. Before describing more precisely the query selection strategy, we recall some definitions about expected regrets \cite{BourdachePS19}.

%(no error is considered in the DM's answers) to our context using \emph{expected} regrets minimization. We first recall some definitions on expected regrets.

\begin{defi}
Given a density function $p$ on $W$ and two solutions $x$ and $y$, the pairwise expected regret is defined as follows:
\[
\per(x,y,p) = \int \max\{0,f_{w}(y)\!-\!f_{w}(x)\}p({w}) d{w}.
\]
\end{defi}

In other words, the Pairwise Expected Regret (PER) of $x$ with respect to $y$ represents the expected utility loss when recommending solution $x$ instead of solution $y$. In practice, the PER is approximated using a sample $S$ of weighting vectors drawn from $p$. This discretization of $W$ enables us to convert the integral into an arithmetic mean: 
\begin{equation}\label{eq:per}
    \per(x,y,S)=\frac{1}{|S|}\sum_{w \in S}\max\{0,f_{w}(y)-f_{w}(x)\}    
\end{equation}

\begin{defi}
Given a set $\mathcal{X}$ of solutions and a density function $p$ on $W$, the max expected regret of $x\!\in\!\mathcal{X}$ and the minimax expected regret over $\mathcal{X}$ are defined by:
\[
\begin{aligned}
\mer(x,\mathcal{X},p)&=\max_{y\in \mathcal{X}} \per(x,y,p),\\
\mmer(\mathcal{X},p)&=\min_{x \in \mathcal{X}} \mer(x,\mathcal{X},p).
\end{aligned}
\]
\end{defi}

Put another way, the max expected regret of $x$ is the maximum expected utility loss incurred in selecting $x$ in $\mathcal{X}$ while the minimax expected regret is the minimal max expected regret value of a solution in $\mathcal{X}$. As for the PER computation, the MER and the MMER can be approximated using a sample $S$ of weight vectors: 
\begin{equation}\label{eq:mer}
\mer(x,\mathcal{X},S)=\max_{y\in \mathcal{X}} \per(x,y,S)
\end{equation}
\begin{equation}\label{eq:mmer}
\mmer(\mathcal{X},S)=\min_{x \in \mathcal{X}} \mer(x,\mathcal{X},S)
\end{equation}
%the value of the minimum worst case loss in $\mathcal{X}$.

We can now describe the adaptation of CSS to the probabilistic setting. The max expected regret of a solution is used to determine which solution to recommend (the lower, the better) in the current state of knowledge characterized by $p$.
%The notion of expected regrets expresses the relative interest of the solutions. Thus, it can be used to evaluate the relevance of selecting one particular solution in $\mathcal{X}$ in the current state of knowledge, and therefore to determine whether to continue the elicitation procedure or not.
At a step $i$ of the elicitation procedure, if the stopping condition (that will be defined below) is met, then a solution $x^{(i)}\!\in\!\arg\min_{x \in \mathcal{X}} \mer(x,\mathcal{X},S)$ is recommended. %relevant recommendation can be defined as the solution $x^{(i)}$ selected in $\arg\min_{x \in \mathcal{X}} \mer(x,\mathcal{X},S)$.
But if the knowledge about the value of $w$ needs to be better specified to make a recommendation, the DM is asked to compare $x^{(i)}$ to its best challenger 
$y^{(i)}\!\in\!\arg\max_{y\in \mathcal{X}} \per(x^{(i)},y,S)$ (best challenger in the current state of knowledge). In the next subsection, we describe how one uses the DM's answer to update the density function $p$.
%a possible good challenger $y^{(i)}$ that we can define as the maximum expected regret of $x^{(i)}$ i.e. $y^{(i)}$ selected in $\arg\max_{y\in \mathcal{X}} \per(x^{(i)},y,S)$. We next explain how we use the DM's answer to update the density function $p$.

\subsection{Bayesian updating}\label{subsec:bayesianUpsating}

At step $i$ of the procedure, a new query of the form ‘‘$x^{(i)} \succsim y^{(i)}?$" is asked to the DM. Her answer is translated into a binary variable $a^{(i)}$ that takes value $1$ if the answer is yes and $0$ otherwise. Using Bayes' rule, the posterior density function reads as follows: 
\begin{equation}\label{eq:bayesRule}
    p(w|a^{(i)}) = \frac{p(w)p(a^{(i)}|w)}{p(a^{(i)})}
\end{equation}
where $p(w)$ is assumed to be multivariate Gaussian (the initialization used for $p(w)$ will be specified in the numerical tests section). The posterior density function $p(w|a^{(i)})$ is hard to compute analytically using Equation \ref{eq:bayesRule}. Indeed, the likelihood $p(a^{(i)}|w)$ follows a Bernoulli distribution and no conjugate prior is known for this likelihood function in the multivariate case. Therefore, one uses a data augmentation method \cite{albert93} that consists in introducing a latent variable $z^{(i)}\!=\!w^Td^{(i)}\!+\!\varepsilon^{(i)}$ that represents the utility difference between the two compared solutions, where $w$ is a given weighting vector, $d^{(i)}$ is an explanatory variable defined by $d^{(i)}\!=\!x^{(i)}\!-\!y^{(i)}$ and $\varepsilon^{(i)}\!\sim\! \mathcal{N}(0,\sigma)$ is a Gaussian noise accounting for the uncertainty about the DM's answer. The Gaussian nature of the density function for $\varepsilon^{(i)}$ implies that the conditional distribution $z^{(i)}|w$ is also Gaussian: $z^{(i)}|w\!\sim\!\mathcal{N}(w^T d^{(i)},\sigma)$. In order to make $z^{(i)}$ consistent with the DM's answer, one forces $z^{(i)}\!\ge\!0$ if $a^{(i)}\!=\!1$ and $z^{(i)}\!<\!0$ otherwise. Thus, one obtains the following truncated density function:
\[
p(z^{(i)}|w, a^{(i)})  \propto \left\{
\begin{array}{ll}
	\mathcal{N}(w^T d^{(i)}, \sigma) \mathds{1}_{z^{(i)} \ge 0} &  \text{if } a^{(i)} = 1\\
	\mathcal{N}(w^T d^{(i)}, \sigma) \mathds{1}_{z^{(i)} < 0} &  \text{otherwise}
\end{array}
\right.
\]
Using the latent variable, the posterior distribution $p(w|a^{(i)})$ is formulated as:
\begin{equation} \label{eq:p_w_ai}
    p(w|a^{(i)})\!=\!\!\!\int\!p(w,z|a^{(i)})dz\!=\!\!\!\int\!p(w|z)p(z|a^{(i)})dz
\end{equation}
If the prior density $p(w)$ is multivariate Gaussian then $p(w|z)$ is multivariate Gaussian too, as well as $p(w|a^{(i)})$. To approximate $p(w|a^{(i)})$, Tanner and Wong proposed in \cite{tanner87} an iterative procedure based on the fact that $p(z|a^{(i)})$ depends in turn on $p(w|a^{(i)})$:
\begin{equation} \label{eq:p_z_ai}
    p(z|a^{(i)})=\int p(\omega|a^{(i)})p(z|\omega,a^{(i)})~d\omega
\end{equation}
The procedure consists in solving the fixed point equation obtained by replacing $p(z|a^{(i)})$ in Equation \ref{eq:p_w_ai} by its expression in Equation \ref{eq:p_z_ai}.
More precisely, ones performs alternately series of samples of $m$ values $z_1, \dots, z_m$ from $z|a^{(i)}$, and updating of the posterior density function by $p(w|a^{(i)})\!=\!\frac{1}{m}\sum_{j=1}^m p(w|z_j)$, which is Gaussian because every $w|z_j, j\!\in\!\llbracket1, m\rrbracket$ is Gaussian. Note that each value of the sample $z_1, \dots, z_m$ is obtained by iteratively drawing a value $w$ from the current distribution $p(w)$ then drawing $z_j$ from $p(z_j|w, a^{(i)})$ for all $j$. For more details, we refer the reader to Algorithm $2$ in \cite{BourdachePS19}.

\subsection{Stopping condition} \label{subsec:stoppingCond}
The principle of the incremental elicitation procedure is to alternate queries and update operations on the density $p(w)$ until the uncertainty about the weighting vector $w$ is sufficiently reduced to be able to make a recommendation with a satisfactory confidence level. A stopping condition that satisfies this specification consists in waiting for the $\mmer(\mathcal{X}, S)$ value to drop below a predefined threshold, which can be defined as a percentage of the initial $\mmer$ value.

%alternatively asking queries to the DM and updating the density function as explained above until we sufficiently reduce the uncertainty about the value of the weighting parameter $w$ to be able to make a recommendation with a satisfactory confidence. A possible stopping condition is to check if the $\mmer(\mathcal{X}, S)$ value converges towards a value below a predefined threshold. This threshold can be defined as a percentage of the initial $\mmer$ value.

\subsection{Main obstacles for extending the approach}%Combinatorial Case: Main Issue}

The main obstacles encountered while managing to extend the approach to a combinatorial setting are related to the computation of MER and MMER values as they are defined in Equations \ref{eq:mer} and \ref{eq:mmer}:
\begin{itemize}
\item both values requires an exponential number of pairwise comparisons to be computed (because there is an exponential number of feasible solutions);
\item in addition, the use of linear programming to compute these values is not straightforward because the constraint $\max\{0,.\}$ in Equation \ref{eq:per} is not linear.
\end{itemize}
These issues are all the more critical given that the MER and MMER values are computed at every step of the incremental elicitation procedure to determine whether it should be stopped or not, and to select the next query.

%The computation of MER and MMER values as it is defined in Equations \ref{eq:mer} and \ref{eq:mmer} is impractical on combinatorial domains, as both values requires an exponential number of pairwise comparisons to be computed. In addition, the use of linear programming is not straightforward because of the $\max\{0,.\}$ constraint that makes the objective functions of MER and MMER non-linear. This is a main issue, especially since they are computed at every step of the algorithm in order to determine whether it should be stopped or not or to select the next query and update the density function. We introduce in the next section a new method of MER and MMER computation to overcome this issue.

% ---------------------------------------------------------------------
% ---------------------------------------------------------------------
% ---------------------------------------------------------------------

\section{Computation of regrets}\label{sec:mmerComputation}

While the use of mathematical programming is standard in minmax regret optimization, the framework of minmax \emph{expected} regret optimization is more novel.
We propose here a new method to compute $\mer(x, \mathcal{X}, S)$ and $\mmer(\mathcal{X}, S)$ by mixed integer linear programming, where $\mathcal{X}$ is implicitly defined by a set of linear constraints and $S$ is a sample drawn from the current density $p(w)$. We consider in this section that $f_w(x)$ is linear in $u(x)$, but the presented approach is adaptable to non-linear aggregation functions if there exists appropriate linear formulations (e.g., the linear formulation of the ordered weighted averages \cite{Ogryczak03}). We also assume that $f_w(x)\!\in\! [0,1]$.

\subsection{Linear programming for MER computation}
To obtain a linear expression for $\mer(x, \mathcal{X}, S)$, we replace the function $\max\{0,f_w(y)-f_w(x)\}$ in Equation \ref{eq:per} by $b_w[f_w(y)-f_w(x)]$ for each weighting vector $w\!\in\!S,$ where $b_w$ is a binary variable such that $b_w\!=\!1$ if $f_w(y)-f_w(x)\!>\!0$ and $b_w\!=\!0$ if $f_w(y)-f_w(x)\!<\!0$ (the value of $b_w$ does not matter if $f_w(y)-f_w(x)\!=\!0$, because $b_w[f_w(y)-f_w(x)]\!=\!0$ anyway). For this purpose, we need the following additional constraints:
\[
    \left\{ 
    \begin{array}{llr}
        b_w \le f_w(y)-f_w(x)+1 & \forall w \in S & \qquad (c_\le)\\
        b_w \ge f_w(y)-f_w(x)   & \forall w \in S & \qquad (c_\ge)
    \end{array}
    \right.
\]
\begin{prop} \label{prop:max0}
    Given $w\!\in\!S$, $x\!\in\!\mathcal{X}$ and $y\!\in\!\mathcal{X}$, if $f_w$ is an aggregation function defined such that $f_w(z)\!\in\![0,1]$ for any $z\!\in\!\mathcal{X}$ and $w\!\in\!S$, and if $b_w$ satisfies the constraints $(c_\ge)$ and $(c_\le)$, then:
    \[\max\{0,f_w(y)-f_w(x)\}=b_w[f_w(y)-f_w(x)].\]
\end{prop}

\begin{proof}
Let us denote by $d_w$ the value $f_w(y)-f_w(x)$ for any $w\!\in\!S$. First note that $d_w\!\in\![0,1], \forall w\!\in\!S$, because $f_w$ is such that $f_w(z)\!\in\![0,1], \forall z\!\in\!\mathcal{X}$. For any $w\!\in\!S$, three cases are possible: \emph{Case 1.} $w$ is such that $d_w\!>\!0$: $(c_\ge)$ becomes $b_w\!\ge\!d_w\!>\!0$, thus $b_w\!=\!1$ and we indeed have $b_wd_w\!=\!d_w\!\ge\!0$. \emph{Case 2.} $w$ is such that $d_w\!<\!0$: $(c_\le)$ becomes $b_w\!\le\!d_w\!+\!1\!<\!1$ and implies $b_w\!=\!0$ and thus $b_wd_w\!=\!0$. \emph{Case 3.} $w$ is such that $d_w\!=\!0$ then $b_wd_w\!=\!0, \forall b_w\in\{0,1\}$. In the three cases we have thus $b_wd_w\!=\!\max\{0,d_w\}$.
\end{proof}

% -------------------------------------------
The constraints $(c_\le)$ and $(c_\ge)$ are linear as $f_w(x)$ is linear in $u(x)\!=\!(u_1(x),\ldots,u_n(x))$. Nevertheless, using variables $b_w$ and their constraints in the formulation of $\mer(x, \mathcal{X}, S)$ gives a system of linear constraints with a quadratic objective function:
\[ 
\begin{array}{ll}
\max \frac{1}{|S|} \sum_{w \in S} b_w [f_w(y) - f_w(x)] \\
    ~~ b_w \le f_w(y)-f_w(x)+1 & \forall w \in S \\
    ~~ b_w \ge f_w(y)-f_w(x) & \forall w \in S \\
    ~~ b_w \in \{0, 1\} & \forall w \in S \\
    ~~ y \in \mathcal{X}
\end{array}
\]
The objective function is quadratic because the term $b_w f_w(y)$ is quadratic in variables $b_w$ and $y$. To linearize the program, we introduce a positive real variable $p_w$ for each $w\!\in\!S,$ that replace the product term $b_w f_w(y)$. Note that the term $b_wf_w(x)$ does not need linearization because solution $x$ is fixed in the MER computation. The obtained linear program is:
\[ 
(P_{\mer}):
\begin{array}{ll}
\max \frac{1}{|S|} \sum_{w \in S} [p_w - b_w f_w(x)] \\
    ~~ b_w \le f_w(y)-f_w(x)+1 & \forall w \in S \\
    ~~ b_w \ge f_w(y)-f_w(x) & \forall w \in S \\
    ~~ p_w \le b_w & \forall w \in S \\
    ~~ p_w \le f_w(y) & \forall w \in S \\
    ~~ p_w \ge b_w + f_w(y) - 1 & \forall w \in S \\
    ~~ b_w \in \{0, 1\} & \forall w \in S \\
    ~~ p_w \in \mathbb{R}^+ & \forall w \in S \\
    ~~ y \in \mathcal{X}
\end{array}
\]
It is easy to see that $p_w=b_wf_w(y)$ for all $w\!\in\!S$ thanks to the constraints on $p_w$. We indeed have $p_w\!=\!0$ when $b_w\!=\!0$ thanks to the constraint $p_w\!\le\!b_w$, and $p_w\!=\!f_w(y)$ when $b_w\!=\!1$ thanks to constraints $p_w\!\le\!f_w(y)$ and $p_w\!\ge\!b_w\!+\!f_w(y)\!-\!1\!=\!f_w(y)$. 

Overall, $2|S|$ variables are involved in the linearization of the expression $\frac{1}{|S|} \sum_{w \in S} \max\{0, [f_w(y) - f_w(x)\}$: $|S|$ binary variables $b_w$ are used to linearize the $\max\{0,.\}$ function, and $|S|$ real variables $p_w$ are used to linearize the product term $b_wf_w(y)$.

\subsection{Linear programming for MMER computation}
For computing $\mmer(\mathcal{X}, S)$, the objective function 
\[
\min_{x\in \mathcal{X}}\max_{y\in \mathcal{X}}\frac{1}{|S|}\sum_{w \in S} \max\{0,f_w(y)\!-\!f_w(x)\}
\]
can be linearized by using $|\mathcal{X}|$ constraints (standard linearization of a $\min\max$ objective function, where the max is taken over a finite set):
%the $|S|$ binary variables $b_w$ must be replaced by $|S|\!\times\!|\mathcal{X}|$ variables $b_w^y$ in order to linearize the $\min_{x \in \mathcal{X}}\max_{y \in \mathcal{X}}$ objective function: %\frac{1}{|S|} \sum_{w \in S} \max\{0,f_w(y)-f_w(x)\}$:% because $y$ can be any feasible solution in $\mathcal{X}$ while it was induced by the choice of $x$ in the computation of $\mer(x, \mathcal{X}, S)$:% P:adding the binary variables yields the following linear program,
\[
\begin{array}{lr}
    \min t \\
    ~~ t \ge \frac{1}{|S|} \sum_{w \in S} \max\{0, f_w(y)\!-\!
    f_w(x)\} ~ \forall y \in \mathcal{X} & (*)\\
    ~~ t \in \mathbb{R}
\end{array}
\]
Note that computing the minmax expected regret over $\mathcal{X}$ requires the introduction of one binary variable $b_w^y$ for \emph{each} solution $y\!\in\! \mathcal{X}$, so that
\[
\max\{0,f_w(y)-f_w(x)\}=b_w^y(f_w(y)-f_w(x))
\]
for all $y \in \mathcal{X}$ (while computing the max expected regret of a \emph{given} solution $x$ only required the introduction of a \emph{single} binary variable $b_w$ such that $\max\{0,f_w(\hat{y})\!-\!f_w(x)\}\!=\!b_w(f_w(\hat{y})\!-\!f_w(x))$ for $\hat{y}\!\in\!\arg\max_{y \in \mathcal{X}}\per(x,y,S)$). 

%Thus obtaining not only an exponential number of additional constraints but also an exponential number of additional variables to linearize the $\max\{0,.\}$ function. 
Let us consider the following program $P_\mmer$, involving quadratic constraints:
\[
\begin{array}{lr}
\min t \\
        ~~ t \ge \frac{1}{|S|} \sum_{w \in S} b^y_w [f_w(y) - f_w(x)] & \forall y \in \mathcal{X}\\
        ~~ b^y_w \le f_w(y)-f_w(x)+1 & \forall w, y \in S \times \mathcal{X} \\
        ~~ b^y_w \ge f_w(y)-f_w(x) & \forall w, y \in S \times \mathcal{X} \\
        ~~ b^y_w \in \{0, 1\} & \forall w, y \in S \times \mathcal{X} \\
        ~~ x \in \mathcal{X} \\
        ~~ t \in \mathbb{R}
\end{array}
\]

\begin{prop}
    A solution $x^*\!\in\!\mathcal{X}$ optimizing $P_\mmer$ is such that:
    \[
    \mer(x^*, \mathcal{X}, S)\!=\!\mmer(\mathcal{X}, S).
    \]
\end{prop}

\begin{proof} \label{prop:QPmmer}
    We denote by $t^*$ the optimal value of $P_\mmer$. We now prove that $t^*$ is equal to $\mmer(\mathcal{X}, S)$. For a given instance of $x$, constraint $(*)$ must be satisfied for \emph{any} possible instance of $y$. Thus, by Proposition~\ref{prop:max0}, we have that $t\!\ge\!\per(x,y,S)$ for all $y\!\in\! \mathcal{X}$ because:
    \[
    \abovedisplayskip = 8pt
    \frac{1}{|S|}\sum_{w \in S} b_w^y [f_w(y)\!-\!f_w(x)]\!=\!\per(x,y,S).
    \belowdisplayskip = 8pt
    \]
    It implies that $t\!\ge\!\max_y \per(x,y,S)\!=\!\mer(x,\mathcal{X},S)$. As the objective function is $\min t$, for each instance of $x$, the variable $t$ takes value $\mer(x,\mathcal{X},S)$.
    %Consequently, minimizing $t$ for $x \in \mathcal{X}$ amounts to minimizing $\mer(x,\mathcal{X},S)$ for $x \in \mathcal{X}$. One concludes that $t^*\!=\!\mmer(\mathcal{X},S)$, and the corresponding $x^*\!\in\! \mathcal{X}$ is such that $\mer(x^*,\mathcal{X},S)=\mmer(\mathcal{X},S)$.
    %we need one binary variable $b^y_w$ for each possible instance of $y$ and then the expression $\frac{1}{|S|} \sum_{w\in S} b^y_w [f_w(y)\!-\!f_w(x)]$ corresponds to $\per(x, y, S)$ using proposition \ref{prop:max0}. Thus, constraint $(*)$ ensures that $t\!\ge\!\per(x, y, S), \forall y\!\in\!\mathcal{X}$, and thus $t\!\ge\!\mer(x, \mathcal{X}, S)$. 
    The $\min$ objective function implies that (1) $t\!=\!\mer(x, \mathcal{X}, S)$ for a given $x$. Finally, varying $x$ over $\mathcal{X}$, we can easily see that $t^*\!\le\!\mer(x, \mathcal{X}, S)\  \forall x\!\in\!\mathcal{X}$, and thus (2) $t^*\!=\!\mmer(\mathcal{X}, S)$. The result follows from (1) and (2).
\end{proof}

%The constraints linearizing the $\max\{0,.\}$ function thus involve $|S|\!\times\!|\mathcal{X}|$ variables $b_w^y$.
%as many constraints of type $(*)$ as there are feasible solutions $y$ in $\mathcal{X}$. 
%Consequently, 
The quadratic terms $b^y_w f_w(x)$ are linearized by introducing $|S|\!\times\!|\mathcal{X}|$ positive real variables $p_w^y$:
\[
(P_\mathcal{X})\!:
\begin{array}{lr}
\min t \\
    ~ t \ge \frac{1}{|S|} \sum_{w \in S} b^y_w [f_w(y) - p^y_w]  & \forall y \in \mathcal{X} \\
    ~ b^y_w \le f_w(y)-f_w(x)+1 & \forall w, y \in S\!\times\!\mathcal{X} \\
    ~ b^y_w \ge f_w(y)-f_w(x) & \forall w, y \in S\!\times\!\mathcal{X} \\
    ~ p^y_w \le b^y_w & \forall w, y \in S\!\times\!\mathcal{X} \\
    ~ p^y_w \le f_w(x) & \forall w, y \in S\!\times\!\mathcal{X} \\
    ~ p^y_w \ge b^y_w + f_w(x) - 1 & \forall w, y \in S\!\times\!\mathcal{X} \\
    ~ b^y_w \in \{0, 1\} & \forall w, y \in S\!\times\!\mathcal{X} \\
    ~ p^y_w \in \mathbb{R}^+ & \forall w, y \in S\!\times\!\mathcal{X} \\
    ~ x \in \mathcal{X} \\
    ~ t\in \mathbb{R}
\end{array}
\]
One comes up with a mixed integer linear program $P_\mathcal{X}$ involving $|S|\!\times\!|\mathcal{X}|$ binary variables $b_w^y$, $|S|\!\times\!|\mathcal{X}|$ positive real variables $p_w^y$ and $|\mathcal{X}|\!+\!6\!\times\!|S|\!\times\!|\mathcal{X}|$ constraints, hence an exponential number of variables and constraints due to the combinatorial nature of the set $\mathcal{X}$. In the remainder of the section, we propose a method to overcome this issue.
%Furthermore, there are $|S|\!\times\!|\mathcal{X}|$ variables $b_w^y$ (here $x$ can be any feasible solution in $\mathcal{X}$, while it is fixed in the computation of $\mer(x, \mathcal{X}, S)$).
%an exponential number of additional variables and constraints are necessary to linearize the $\max\{0, .\}$ function. Indeed, in order to satisfy the constraint $(*)$ for all $y\in\mathcal{X}$, a set of binary variables $\{b^y_w, \forall w \in S\}$ is necessary per instance of $y$, because in this case both $x$ and $y$ are variables. Consequently, an exponential number of positive real variables are necessary to linearize the products $b^y_w f_w(.)$. 
%Here an exponential number of constraints of type $(*)$ are necessary to linearize the $\min\max$ objective function: as many constraints as there are feasible solutions in $\mathcal{X}$. And an exponential number of additional variables and constraints are necessary to linearize the $\max\{0, .\}$ function. Indeed, in order to satisfy the constraint $(*)$ for all $y\in\mathcal{X}$, a set of binary variables $\{b^y_w, \forall w \in S\}$ is necessary per instance of $y$, because in this case both $x$ and $y$ are variables. Consequently, an exponential number of real positive variables are necessary to linearize the products $b^y_w f_w(.)$. The remaining part of the section introduces a new method to overcome this issue.

\subsection{MMER computation method}
\label{subsec:mmercomputation}

The proposed method is based on mixed integer linear programming with dynamic generation of variables and constraints to compute $\mmer(\mathcal{X}, S)$, an optimal solution $x_S^*\!\in\!\arg\min_{x\in\mathcal{X}}\mer(x, \mathcal{X}, S)$ and its best challenger $\hat{y}_S\!\in\!\arg\max_{y\in\mathcal{X}}\per(x_S^*, y, S)$.

Let us first define a mixed integer linear program $P_A$ that contains only a subset of variables $b^y_w$ and $p^y_w$, and a subset of constraints of type $(*)$. Given a subset $A\!\subseteq\!\mathcal{X}$ of solutions, $P_A$ computes the minimax expected regret $\mmer_A(\mathcal{X}, S)$ defined by:
\[
\displaystyle\min_{x \in \mathcal{X}} \mer(x, A, S)\!\!=\!\!\min_{x\in\mathcal{X}} \max_{y\in A} \per(x,y,S).
\]
Put another way, $\mer(x, A, S)$ is the max expected regret of a solution $x\!\in\!\mathcal{X}$ w.r.t. solutions in $A$. More formally, $P_A$ is written as follows:
%only the solutions in $A$ are considered as possible challengers when computing the max regret of a solution $x\!\in\!\mathcal{X}$. 
\[
(P_A)\!:
\begin{array}{lr}
\min t \\
    ~ t \ge \frac{1}{|S|} \sum_{w \in S} b^y_w [f_w(y) - p^y_w]  & \forall y \in A \\
    ~ b^y_w \le f_w(y)-f_w(x)+1 & \forall w, y \in S\!\times\!A \\
    ~ b^y_w \ge f_w(y)-f_w(x) & \forall w, y \in S\!\times\!A \\
    ~ p^y_w \le b^y_w & \forall w, y \in S\!\times\!A \\
    ~ p^y_w \le f_w(x) & \forall w, y \in S\!\times\!A \\
    ~ p^y_w \ge b^y_w + f_w(x) - 1 & \forall w, y \in S\!\times\!A \\
    ~ b^y_w \in \{0, 1\} & \forall w, y \in S\!\times\!A \\
    ~ p^y_w \in \mathbb{R}^+ & \forall w, y \in S\!\times\!A \\
    ~ x \in \mathcal{X} \\
    ~ t \in \mathbb{R}
\end{array}
\]
Note that $P_A$ now only involves $|S|\!\times\!|A|$ variables $b^y_w$, $|S|\!\times\!|A|$ variables $p^y_w$ and $|A|\!+\!6\!\times\!|S|\!\times\!|A|$ constraints.
%requires $|S|\!times\!|A|$ variables, additional variables and $|A|$ constraints $(*)$ for the linearization.

\medskip
The algorithm we propose consists in alternatively solving $P_A$ and $P_\mer$.
%, adding the solutions $\hat{y}$ returned by $P_\mer$ to the set $A$, and consequently adding the corresponding variables and constraints to $P_A$, until 
Let $x_A$ (resp. $\hat{y}$) denote the optimal solution returned by solving $P_A$ (resp. $P_{\mer}$ for $x\!=\!x_A$). The algorithm starts with a small set $A$ of feasible solutions (see Section~\ref{subsec:IncrementalApproach} for details regarding the initialization), and then iteratively grows this set by adding to $A$ the best challenger $\hat{y}$ of $x_A$. Convergence is achieved when $P_\mer$ returns a solution $\hat{y}$ that already belongs to $A$, which implies that $\mmer_A(\mathcal{X}, S)\!=\!\mmer(\mathcal{X}, S)$. Algorithm \ref{algo:mmer} describes the procedure.

\begin{algorithm}[h!] 
\caption{$\mmer(\mathcal{X}, A, S)$}
\label{algo:mmer}

\KwIn{$\mathcal{X}$: combinatorial set of feasible solutions;\\ \hspace{0.95cm}$A \subseteq \mathcal{X}$: subset of challengers;\\
\hspace{0.95cm}$S$: sample of weighting vectors.}
\KwOut{$\mmer$ value, $\mmer$ solution and its best challenger for the considered sample}

$\hat{y} \leftarrow null$ \\
\Repeat{$\hat{y}\in A$}{
  \lIf{$\hat{y}\neq null$}{$A \leftarrow A \cup \{\hat{y}\}$} 
  $(\text{\emph{mmer}}_A, x_A) \leftarrow$ $\mmer_A(\mathcal{X}, S)$ (using $P_A$) \\
  %$y^* \leftarrow \arg \max_{y \in \mathcal{X}} \per(x_A, y, S)$ (using $P_\mer$) \label{algoLine:challenger}\\
  $(\text{\emph{mer}\_}\,x_A,\hat{y}) \leftarrow \mer(x_A, \mathcal{X}, S)$ (using $P_\mer$) %\max_{y \in \mathcal{X}} \per(x_A, y, S)$ 
  \label{algoLine:challenger}
}
%\emph{mmer}, $x^*_S$, $y^*_S\leftarrow$ \emph{mmer}$_A$, $x_A$, $\hat{y}$

\Return \emph{mmer}$_A$, $x_A$, $\hat{y}$%\emph{mmer}, $x^*_S$, $y^*_S$.
\end{algorithm}

By abuse of notation, $\mmer_A(\mathcal{X}, S)$ is viewed in the algorithm as a procedure returning the couple consisting of the optimal value \emph{mmer}$_A$ of $P_A$ and the corresponding optimal solution $x_A$. Similarly, $\mer(x_A, \mathcal{X}, S)$ is viewed as a procedure returning the couple consisting of the optimal value \emph{\text{mer}\_$\,x_A$} of $P_\mer$ and the corresponding optimal solution $\hat{y}$.
%of $P_A$ (resp. \emph{\text{mer}\_$\,x_A$})  (resp. $P_\mer$)  (resp. $y^*$). %Thus, \emph{mmerCurrentVal} and $x_A$ are, respectively, the value of $\mmer_A(\mathcal{X}, S)$ and the corresponding solution returned by $P_A$. Similarly, \emph{\text{mer}\_$\,x_A$} and $y^*$ are, respectively, the value of $\mer(x_A, \mathcal{X}, S)$ and the corresponding solution returned by $P_\mer$. 
At the termination of the algorithm, \emph{mmer}$_A$ corresponds to $\mmer(\mathcal{X}, S)$ and $x_A$ is the MMER solution (and $\hat{y}$ its best challenger).
%the returned values \emph{mmerCurrentVal} and $x_A$ corresponds to, respectively, the $\mmer(\mathcal{X}, S)$ value and the MMER corresponding solution $x^*\!=\!\min_{x\in\mathcal{X}}\mer(x, \mathcal{X}, S)$.

\begin{prop}
    Algorithm \ref{algo:mmer} terminates and returns a minmax expected regret solution and its best challenger. \label{prop:validiteAlgo}
\end{prop}

\begin{proof}

First, it is easy to see that Algorithm \ref{algo:mmer} always terminates. Indeed, at every step of the algorithm, if the stopping condition is not satisfied then a new solution $\hat{y}\!\not\in\!A$ is added to $A$ and a new iteration is performed. In the worst case, all the solutions of $\mathcal{X}$ are added to the set $A$ and  the stopping condition is trivially satisfied. %Consequently, if at step $k$ of the algorithm $A\!=\!\mathcal{X}$ then at step $k\!+\!1$ the stopping condition is trivially satisfied.

We now prove the validity of Algorithm \ref{algo:mmer}, i.e.:
\[
\mmer_A(\mathcal{X}, S)\!=\!\mmer(\mathcal{X}, S) \mbox{ if } \hat{y}\!\in\!A.
\]
Assume that $A\!\subsetneq\!\mathcal{X}$ (if $A\!=\!\mathcal{X}$ the equality $\mmer_A(\mathcal{X},S)\!=\!\mmer(\mathcal{X},S)$ is trivially true).
On the one hand, at any step of the algorithm, we have (1) $\emph{mmer}_A\!\le$ \emph{mer\_$\,x_A$} because $\mer(x_A,A,S)\!\le\!\mer(x_A,\mathcal{X},S)$. 
On the other hand, if $\hat{y}\!\in\!A$ then the constraint $t\!\ge\!\frac{1}{|S|} \sum_{w \in S} [f_w(\hat{y}) - f_w(x_A)]$ is satisfied  for $t\!=$ \emph{mmer}$_A$, i.e., $\emph{mmer}_A\!\ge\!\per(x_A, \hat{y}, S)$. 
As $\per(x_A, \hat{y}, S)\!=\!\mer(x_A, \mathcal{X}, S)$ by definition of $\hat{y}$, it implies that (2) $\emph{mmer}_A\!\ge\!\mer(x_A, \mathcal{X}, S)=$ \emph{mer\_$\,x_A$}. 
By (1) and (2), we conclude that \emph{mmer}$_A$ $=\emph{\text{mer}}$\_$\,x_A$.

Finally, \emph{mmer}$_A$ minimizes $\frac{1}{|S|}\sum_{w \in S} b_w [f_w(\hat{y}) - f_w(x)]$ for $x \in \mathcal{X}$, thus \emph{mmer}$_A$ minimizes $\per(x, \hat{y}, S)$ over $\mathcal{X}$. Consequently, \emph{\text{mer}\_$\,x_A$} $\le \per(x, \hat{y}, S)$ for all $x \in \mathcal{X}$ and then \emph{\text{mer}\_$\,x_A$} $\le \mer(x, \mathcal{X}, S), \forall x \in \mathcal{X}$. Thus, by definition of the MMER, we have \emph{\text{mer}\_$\,x_A$} $=\mmer(\mathcal{X}, S)$ and thereby $\emph{mmer}_A\!=\!\emph{mer}$\_$\,x_A\!=\!\mmer(\mathcal{X}, S)$.
\end{proof}

\subsection{Clustering the samples}
\label{subsec:clustering}
In order to decrease the computation times between queries, we propose to reduce the number of variables and constraints in $P_A$ by applying a clustering method on each sample $S$ drawn from density $p(w)$.
%For illustration, in Figure \ref{fig:clustering}, a uniform sample of $200$ three-dimensional weighting vectors $w$ in the simplex ($\sum_i w_i\!=\!1$) is partitioned into $5$ clusters. As the weighting vectors are normalized here, only 2 components are sufficient to characterize them ($w_1$ on the x-axis, $w_2$ on the y-axis, and $w_3\!=\!1\!-\!w_1\!-\!w_2$). In this figure, every cluster is represented using a different marker and the big markers represents the clusters centers.
Let $C$ denote the set of cluster centers. The idea is to replace the $|S|$ weights by the $|C|$ cluster centers, the formula for the pairwise expected regret becoming:
\begin{equation}
    \per(x,y,C)=\sum_{c \in C}\rho_c\max\{0,f_{c}(y)-f_{c}(x)\}    
\end{equation}
\noindent where $\rho_c$ is the weight of the cluster center $c\!\in\!C$ and represents the proportion of weighting vectors of $S$ that are in the cluster of center $c$. The formulas for $\mer(x,\mathcal{X},C)$ and $\mmer(\mathcal{X},C)$ are adapted in the same way.

%The weighting vectors $w$ of $S$ are clustered according to their relative positions using a clustering method (\emph{K-Means} for instance) and a set $C$ of cluster centers is computed. Every weighting vector of $S$ is then approximated by the center of its corresponding cluster in $C$.  %in Algorithm \ref{algo:mmer}

%In this way, the sample $S$ in the algorithm can be replaced by the set of cluster centers $C$ and the Pairwise Expected Regret becomes:
%\begin{equation}
%    \per(x,y,C)=\sum_{c \in C}\rho_c\max\{0,f_{c}(y)-f_{c}(x)\}    
%\end{equation}

%\noindent where $\rho_c$ is the weight of the cluster center $c\in C$ and represents the proportion of weighting vectors of $S$ that are in the cluster which center is $c$. $\mer(x,\mathcal{X},C)$ and $\mmer(\mathcal{X},C)$ are defined in the same way.

% \begin{figure}
%     \centering
%     \includegraphics[scale=0.4]{Images/Clusters_ex.png}
%     \caption{Example of clustering in the simplex with $5$ clusters.}
%     \label{fig:clustering}
% \end{figure}

\subsection{Incremental decision making approach}
\label{subsec:IncrementalApproach}

As detailed in section \ref{sec:incremental}, the MMER computation is used to determine which query to ask at each step as well as to trigger the stopping condition. The whole incremental decision making procedure is summarized in Algorithm \ref{algo:eliciation}.
The set $A$ is heuristically defined as the set of $f_w$-optimal solutions for $w$ in $C$ (Line \ref{algoLine:defA}) but can be defined otherwise without any impact on the result in proposition \ref{prop:validiteAlgo}.
The variable \emph{mmer} (Line \ref{algoLine:mmer}) represents the current minmax expected regret value and is computed using Algorithm \ref{algo:mmer} by replacing the sample $S$ by the set of cluster centers $C$.  %The general approach is summarized in Algorithm \ref{algo:eliciation}:
%included in the general algorithm to define the stopping condition as well as the query to ask at each step.

\begin{algorithm}
  \DontPrintSemicolon
  \caption{Incremental Decision Making}
  \label{algo:eliciation}

  \KwIn{$\mathcal{X}$: combinatorial set of feasible solutions;\\
  \hspace{0.95cm}$p_0(w)$: prior density function.}
  \KwOut{$x^*$ : recommended solution.}%WS near-optimal knapsack.}
  
  $p(w) \leftarrow p_0(w); i \leftarrow 1$\\
  \Repeat{mmer \emph{stabilizes}}{
    $S \leftarrow$ sample drawn from $p(w)$ \label{algoLine:sampling} \\
    $C \leftarrow$ cluster centers of $S$ \\
    $A \leftarrow \{\arg \max_{x \in \mathcal{X}} f_w(x)|w \in C\}$ \label{algoLine:defA}\\
    $(\text{\emph{mmer}}, x^{(i)}, y^{(i)}) \leftarrow \mmer(\mathcal{X}, A, C)$ \label{algoLine:mmer}\\
    Ask the DM if $x^{(i)}$ is preferred to $y^{(i)}$ \\
    $a^{(i)} \leftarrow 1$ if the answer is yes and $0$ otherwise \\
    $p(w) \leftarrow p(w|a^{(i)})$ ~~~~~(see subsection \ref{subsec:bayesianUpsating}) \label{update}\\
    $i \leftarrow i + 1$
  }
  \Return $x^*$ selected in $\arg\min_{x \in \mathcal{X}} \mer(x , \mathcal{X}, C)$
\end{algorithm}

% ---------------------------------------------------------------------
% ---------------------------------------------------------------------
% ---------------------------------------------------------------------

\section{Experimental results}

Algorithm \ref{algo:eliciation} has been implemented in Python using the \emph{SciPy} library for Gaussian sampling, the \emph{Scikit-Learn} library for the clustering operations\footnote{We used \emph{k-means} clustering.} and the \emph{gurobipy} module for solving the mixed integer linear programs. The numerical tests have been carried out on $50$ randomly generated instances of the multi-objective \emph{knapsack} and \emph{allocation} problems.  %\emph{multi-criteria allocation problem} and $50$ randomly generated instances of the \emph{multi-objective knapsack problem}. 
For all tests we used an Intel(R) Core(TM) i7-4790 CPU with 15GB of RAM.

\paragraph{Multi-objective Knapsack Problem (MKP)}

This vector optimization problem is formulated as $\max z\!=\!Ux$ subject to $\sum_{i=1}^p\alpha_i x_i\le \gamma$, where $U$ is an $n\!\times\!p$ matrix of general term $u_{ki}$ representing the utility of item $i\!\in\!\{1,\ldots,p\}$ w.r.t objective $k\!\!\in\!\!\{1,\ldots,n\}$, $x\!\!=\!\!(x_1,\ldots,x_p)^T$ is a vector of binary decision variables such that $x_i\!=\!1$ if item $i$ is selected and $x_i\!=\!0$ otherwise, $\alpha_i$ is the weight of item $i$ and $\gamma$ is the knapsack's capacity. The set of feasible knapsacks is $\mathcal{X}\!=\!\{x\!\in\!\{0,1\}^p | \sum_{i=1}^{p} \alpha_i x_i\!\le\!\gamma\}$, and the performance vector $z\!\in\!\mathbb{R}^n$ associated to a solution $x$ is $z\!=\!Ux$. %where $U$ is an $n\times p$ matrix of general term $u_i^k$ representing the utility of item $k$ w.r.t objective $i$. 

To simulate elicitation sessions, we consider the problem $\max_{x\in\mathcal{X}} f_w(x)$ where $f_w(x)\!=\!\sum_kw_k\sum_iu_{ki}x_i$. The weighting vector $w$ in $W=\{w\!\in\![0,1]^n:\sum_kw_k\!=\!1\}$ is initially unknown.
We generated instances of MKP for $n\!=\!5$ objectives and $p\!=\!100$ items. 
Every item $i$ has a positive weight $\alpha_i$ uniformly drawn in $\{1,\ldots,20\}$, and $\gamma\!=\!\frac{1}{2}\sum_{k=1}^{100}\alpha_k$. Utilities $u_{ki}$ are uniformly drawn in $[0, \frac{1}{p}]$ to make sure that $f_w(x)\!\in\![0,1], \forall x\!\in\!\mathcal{X}$.

\paragraph{Multi-objective Allocation Problem (MAP)} 

Given $m$ agents, $r\!<\!m$ \emph{shareable} resources, and $b$ a bound on the number of agents that can be assigned to a resource, the set $\mathcal{X}$ of feasible allocations of resources to agents consists of binary matrices $X$ of general term $x_{ij}$ such that $\sum_{j=1}^rx_{ij}\!=\!1, \forall i\!\in\!\{1,\ldots,m\}$ and $\sum_{i=1}^mx_{ij}\!\le\!b, \forall j\!\in\!\{1,\ldots,r\}$, where $x_{ij}$ are decision variables such that $x_{ij}\!=\!1$ if agent $i$ is assigned resource $j$, and $x_{ij}\!=\!0$ otherwise. The cost of an allocation $x$ w.r.t. criterion $k$ is defined by $z_k\!=\!\sum_{i=1}^m\sum_{j=1}^r c_{ij}^kx_{ij}$, where $c_{ij}^k$ is the cost of assigning agent $i$ to resource $j$ w.r.t. criterion $k$.

%this vector optimization problem is defined by $\min (z_1,\ldots,z_n)$ , where $z_k=\sum_{i=1}^m u_{ij}^kx_{ij}$ and ,  where $x_{ij}$ are decision variables such that $x_{ij}\!=\!1$ if town $i$ is associated to a resource located in town $j$, and $x_{ij}\!=\!0$ otherwise, .

%In this problem, the set of feasible allocations is defined by $\mathcal{X}\!=\!\{x\!\in\!\{0,1\}^p | x_{ij}\!\in\!\{0, 1\}, \sum_{j=1}^rx_{ij}\!=\!1 \text{ and } \sum_{i=1}^rx_{ij}\!\le\!b\}$.

We consider the problem $\min_{x\in\mathcal{X}} f_w(x)$ where $f_w(x)\!=\!\sum_k w_k\sum_i\sum_jc_{ij}^kx_{ij}$ and $w\!\in\!W$ is initially unknown. 
For the tests, we generated instances with $n\!=\!5$ criteria, $m\!=\!50$ agents, $r\!=\!5$ resources and a bound $b\!=\!15$ on the number of agents that can be assigned to a resource. The values $c_{ij}^k$ are randomly generated in $[0,20]$, then normalized ($c_{ij}^k/\sum_i\sum_j c_{ij}^k$) to ensure that $f_w(x)\!\in\![0,1], \forall x\!\in\!\mathcal{X}$.

\paragraph{Simulation of the DM's answers} In order to simulate the interactions with the DM, for each instance, the hidden weighting vectors $w$ are uniformly drawn in the canonical basis of $\mathbb{R}^n$ (the more vector $w$ is unbalanced, the worse the initial recommendation). %set of extreme points of $W=\{w\!\in\![0,1]^n:\sum_kw_k\!=\!1\}$.
%two types of hidden weighting vectors $(w_1,\ldots,w_5)$ are randomly generated:
%\begin{enumerate}
%    \item weighting vectors $w_u$ uniformly drawn over the simplex,
%    \item weighting vectors $w_\text{ext}$ uniformly drawn in the set of extreme points of the simplex (unbalanced weighting vectors).
%\end{enumerate}
At each query, the answer is obtained using the response model given in Section \ref{subsec:bayesianUpsating}, i.e., for query $i$, the answer depends on the sign of $z^{(i)}\!=\!w^Td^{(i)}\!+\!\varepsilon^{(i)}$, where $\varepsilon^{(i)}\!\sim\!\mathcal{N}(0, \sigma^2)$. We used different values of $\sigma$ to evaluate the tolerance of the approach to wrong answers. We set $\sigma\!=\!0$ to simulate a DM that is perfectly reliable in her answers. The strictly positive values are used to simulate a DM that may be inconsistent in her answers. For MKP (resp. MAP), setting $\sigma\!=\!0.01$ led to $16\%$ (resp. $14\%$) of wrong answers, while $\sigma\!=\!0.02$ led to $24\%$ (resp. $21\%$) of wrong answers.

%For \emph{MAP}, the value $\sigma\!=\!0.01$ led to, respectively, $24\%$ and $14\%$ of wrong answers hidden weights of type $1$ and $2$, while value $\sigma\!=\!0.02$ led to, respectively, $31\%$ and $21\%$ of wrong answers for weights of type $1$ and $2$. Regarding the \emph{MKP}, $\sigma\!=\!0.01$ led to, respectively, $30\%$ and $16\%$ of wrong answers for  weights of type $1$ and $2$, while value $\sigma\!=\!0.02$ led to, respectively, $39\%$ and $24\%$ of wrong answers for weights of type $1$ and $2$.

\paragraph{Parameter settings in algorithms}
The prior density in Algorithm~\ref{algo:mmer} is set to $\mathcal{N}((10, \dots 10)^T, 100I_5)$, where $I_5$ is the identity matrix $5\!\times\!5$, so that the distribution is rather flat. At each step of Algorithm~\ref{algo:eliciation}, a new sample $S$ of 100 weighting vectors is generated; the vectors $w\!\in\!S$ are normalized and partitioned into $20$ clusters. This number of clusters has been chosen empirically after preliminary numerical tests: considering the entire sample or using more than 20 clusters led to higher computation times and did not offer a significant improvement on the quality of the recommendations. Last but not least, we stopped the algorithm after 15 queries if the termination condition was not fulfilled before.

% of weighting vectors. are size $|S|\!=\!100$ is generated and partitioned into $20$ clusters of weighting vectors. Note that, at every step, we normalize every weighting vector $w\!\in\!S$ for the MMER and MER computations.

\paragraph{Illustrative example} Before coming to the presentation of the numerical results, let us first illustrate the progress of the elicitation procedure on the following example: we applied Algorithm \ref{algo:eliciation} on a randomly generated instance of MKP with $3$ agents, $100$ items, a hidden weighting vector $w\!=\!(0, 1, 0)$, and we set $\sigma\!=\!0.02$, which led to an error rate of $20\%$. Figure \ref{fig:samples} illustrates the convergence of the generated samples of weighting vectors (Line \ref{algoLine:sampling} of Algorithm~\ref{algo:eliciation}) toward the hidden weight during the execution of the algorithm. As the weighting vectors are normalized, two components are enough to characterize them. Every graph shows the sample drawn at a given step of the algorithm: before starting the elicitation procedure (Query 0), after query $3$ and query $10$. 
\begin{figure}[hbtp]
    \centering
    \begin{tabular}{ccc}
        \includegraphics[scale=0.25]{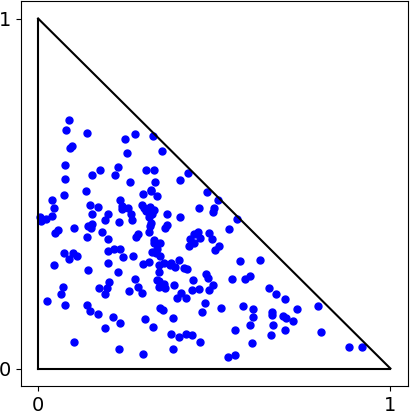} &  
        \includegraphics[scale=0.25]{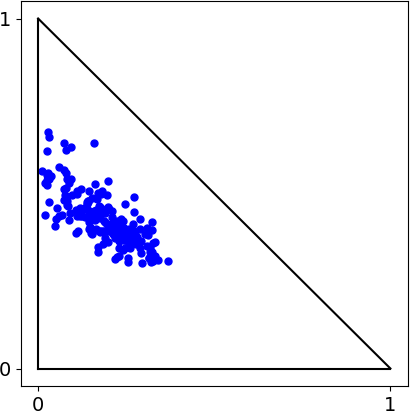} &  
        \includegraphics[scale=0.25]{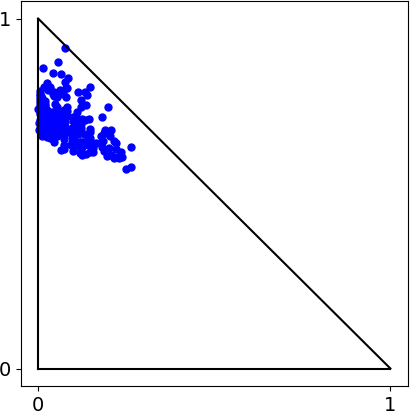} \\
        Query 0 & Query 3 & Query 10
    \end{tabular}
    \caption{Evolution of the samples toward the hidden weight.}%ing vector.}
    \label{fig:samples}
\end{figure}

\paragraph{Analysis of the results} 
We first evaluated the efficiency of Algorithm \ref{algo:eliciation} according to the value of $\sigma$. %and the two different types of hidden weighting vectors. 
We observed the evolution of the quality of the recommendation (the minimax expected regret solution) after every query. The quality of a recommendation $x^*$ is defined by the score $s_{w_h}(x^*)\!=\!f_{w_h}(x^*)/f_{w_h}(x_h)$ for MKP and by $s_{w_h}(x^*)\!=\!\frac{1-f_{w_h}(x^*)}{1-f_{w_h}(x_h)}$ for MAP, where $w_h$ is the hidden weighting vector and $x_h$ is an optimal solution for $w_h$.

%by the score $\frac{f_{w_h}(x^*)-\text{worst}_{w_h}(x^*)}{\max_{x \in \mathcal{X}} f_{w_h}(x)-\text{worst}_{w_h}(x^*)}$, $\text{worst}_{w_h}(x^*)$ is an heuristic function giving the worst solution for weight $w_h$ and is defined by $\text{worst}_{w_h}(x^*)\!=\!\min_{x \in \mathcal{X}} f_{w_h}(x)$ under the constraint $\sum_{k=1}^{100} \alpha_k x_k\!\ge\!\max_{k\!\in\!\llbracket1,100\rrbracket}\alpha_k$. 

The obtained curves for MKP are given 
%in Figure~\ref{fig:score_unif} (for uniform weights $w_u$) and 
in Figure~\ref{fig:score_ext}. %for extreme weights $w_\text{ext}$.
%Note that the results for uniform weights in the case of the MKP are not as much significant as for extreme weights, because in this case, a large part of the generated weights are located around the center part of the simplex, the mean of the prior density function (centered on the uniform vector $(10, \dots, 10)$) is then close to $w_u$ and gives a good heuristic for the determination of an optimal solution for $f_{w_u}$, the score $s_{w_u}$ in the beginning of the algorithm is then already high and can not be significantly improved during the algorithm. Thus, we omitted the curves for $w_u$ here.
We observe that the quality of the recommendation (measured by the score function $s_{w_h}$) is of course negatively impacted when $\sigma$ increases. However, the score of the recommendation at the termination of Algorithm~\ref{algo:eliciation} is $\ge0.98$ for $\sigma\!\in\!\{0, 0.01\}$, and $\ge\!0.96$ for $\sigma\!=\!0.02$. Regarding the computation times, the mean time between two queries over the $50$ instances was around $4$ seconds.%$2$ seconds for $w_u$, $4$ seconds for $w_\text{ext}$.

% \multicolumn{2}{c}{Multi-column}

\begin{table*}[hbtp]
\begin{tabular}{ccc}
    \begin{minipage}{0.4\textwidth}
        \centering\includegraphics[scale=0.45]{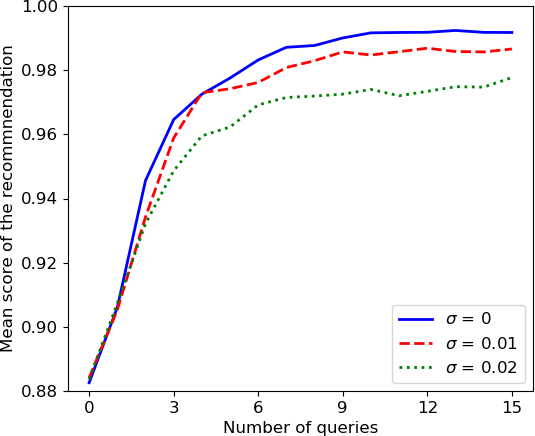}
        \captionof{figure}{\label{fig:score_ext}MKP - Mean score vs. queries}
    \end{minipage}
&
~
&
    \begin{minipage}{0.4\textwidth}
        \centering\includegraphics[scale=0.475]{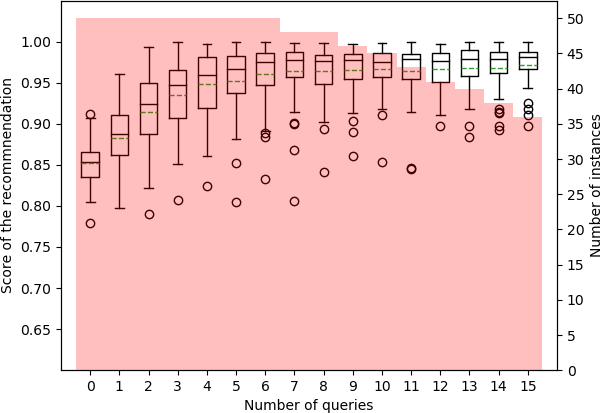}
        \captionof{figure}{\label{fig:box_prob}MKP - Algorithm \ref{algo:eliciation} - Score vs. queries}
    \end{minipage}\\
& \\
    \begin{minipage}{0.4\textwidth}
        \centering
            \includegraphics[scale=1.29]{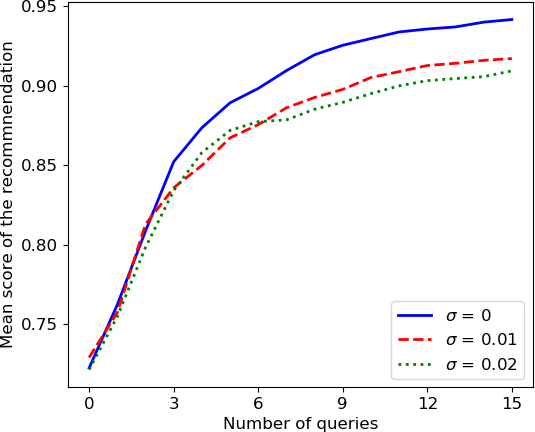}
        \captionof{figure}{\label{fig:score_alloc}MAP - Mean score vs. queries}
    \end{minipage}
&
~
&
    \begin{minipage}{0.4\textwidth}
        \centering\includegraphics[scale=0.475]{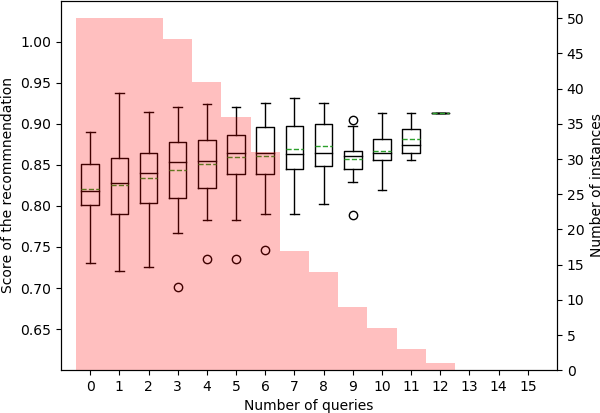}
        \captionof{figure}{\label{fig:box_det}MKP - Deterministic algorithm \cite{BourdacheP19} - Score vs. queries}
    \end{minipage}
\end{tabular}
\end{table*}

Concerning RAP, the curves are given in Figure \ref{fig:score_alloc}. %for both uniformly generated weights $w_u$ (left) and extreme weights $w_\text{ext}$ (right). 
As for MKP, we observe a negative impact on the quality of the recommendation when $\sigma$ increases. Yet, the score of the recommendation is, in the worst case ($\sigma\!=\!0.01$), around $0.86$ from query $4$. %For $w_\text{ext}$, 
The algorithm converges very quickly for all $\sigma$ values, which may be explained by the fact that, for extreme weights $w$, there is a large number of $f_{w}$-optimal solutions; $w$ is indeed such that $w_{i}\!=\!1$ for a given $i$ and all other components takes value $0$, thus, any assignment $x$ such that all the agents are assigned to resources other than resource $i$ are such that $f_{w}(x)\!=\!0$.
Regarding the computation times, the mean time between two queries over the $50$ instances was around $0.9$ seconds for $\sigma\!\in\!\{0, 0.05\}$, and around $1.7$ seconds for $\sigma\!=\!0.1$.

%the mean time between two queries over the $50$ instances was around $0.4$ (resp. $0.9$) seconds for $\sigma\!\in\!\{0, 0.05\}$ and around $1$ (resp. $1.7$) seconds for $\sigma\!=\!0.1$ in the case of $w_u$ (resp. $w_\text{ext}$).

%\medskip
Finally, we compared the performances of Algorithm \ref{algo:eliciation} on MKP to the performance of a deterministic approach that does not take into account the possible errors in responses \cite{BourdacheP19} (approach based on the systematic reduction of the parameter space by minimizing the minimax regret at each step). The aim was to evaluate how much the DM's inconsistencies in her answers impact the two procedures. 
In this purpose, we set $\sigma\!=\!0.02$. %and a hidden weight of type $2$ for which the errors in the DM's answers can be more prohibitive than for type $1$. 
The obtained results are given in the box plots of Figure \ref{fig:box_prob} for Algorithm \ref{algo:eliciation}, and of Figure \ref{fig:box_det} for the deterministic algorithm. In these figures, the box plots give, for any given question, the score of the recommendation for every considered instance (the bottom and top bands of the whiskers are the minimum and maximum scores over the 50 instances, the bottom and top bands of the boxes are the first and third quartiles, the band in the box is the median, the dotted band is the average, and the circles are isolated values). The histogram gives the number of observed values for every query; the bin $i$ indeed gives the number of instances for which query $i$ is reached before the stopping condition is fulfilled.%t least $i$ iterations (queries) of the algorithm are performed.

Figures \ref{fig:box_prob} and \ref{fig:box_det} show the interest of considering our Bayesian elicitation procedure in comparison with a deterministic approach. Indeed, the deterministic approach converges quickly and requires less queries than Algorithm \ref{algo:eliciation}; however, the score of the current recommendation at every step of the algorithm does not exceed $0.94$ for any considered instance and is $\le\!0.9$ for $75\%$ of the instances. In contrast, for Algorithm \ref{algo:eliciation}, the score of the current recommendation is $\ge\!0.95$ in $75\%$ of the instances from query $6$.%Algorithm \ref{algo:eliciation} returns a recommendation with a score $\ge\!0.95$ in 

% ---------------------------------------------------------------------
% ---------------------------------------------------------------------
% ---------------------------------------------------------------------

\section{Conclusion}

We introduced in this paper a Bayesian incremental preference elicitation approach for solving multiobjective combinatorial optimization problems when the preferences of the decision maker are represented by an aggregation function whose parameters are initially unknown. The proposed approach deals with the possibility of inconsistencies in the decision maker's answers to pairwise preference queries. Our approach uses a columns and constraints generation solution method for the computation of expected regrets. The approach is general and can be applied to any problem having an efficient mixed integer linear programming formulation. An interesting research direction would be to refine the approach in the case of non-linear aggregation functions. The approach is indeed compatible with such aggregation functions provided they can be linearized (e.g., the linearization of the ordered weighted averages \cite{Ogryczak03}), but the subsequent linear formulations often involve many additional variables and constraints, thus the need for an optimization.%that would require to optimize our method.%result in an increase in the number of generated variables and constraints in our method, which can be computationally impractical. 

\bibliographystyle{apalike}
\bibliography{biblio}

\end{document}